\documentclass[runningheads]{llncs}
\usepackage{amssymb,amsmath}
\usepackage{graphicx}
\usepackage[misc]{ifsym}
\usepackage{booktabs,siunitx,etoolbox,threeparttable}
\usepackage{subcaption}
\usepackage{microtype}
\usepackage[colorlinks=true,citecolor=blue,urlcolor=black]{hyperref}
\usepackage[capitalise]{cleveref}

\usepackage{amsmath,amsfonts,amssymb}
\usepackage{mathtools}
\usepackage{xparse}

\DeclarePairedDelimiter{\tracefences}{[}{]}
\newcommand{\trace}{\operatorname{tr}\tracefences}
\newcommand{\inv}{^{-1}}
\renewcommand{\det}[1]{\left|#1\right|}
\newcommand{\cond}{\mathbin{|}}
\newcommand{\midcond}{\,\middle|\,}
\DeclarePairedDelimiter{\norm}{\lVert}{\rVert}
\DeclarePairedDelimiter{\innerprod}{\langle}{\rangle}
\DeclarePairedDelimiter{\abs}{\lvert}{\rvert}
\newcommand{\diff}{\mathop{}\!\mathrm{d}}
\newcommand{\Diff}{\mathop{}\!\mathrm{D}}

\NewDocumentCommand\prob{somg}{
	{p
	\IfValueT{#2}%
		{_{#2}}%
	\IfBooleanTF{#1}
		{%
			\!\left(#3%
			\IfValueT{#4}%
				{\midcond #4}%
			\right)
		}{%
			(#3%
			\IfValueT{#4}%
				{\cond #4}%
			)%
		}%
	}%
}

\NewDocumentCommand\expec{somg}{
	{\mathbb{E}%
	\IfValueT{#2}%
		{_{#2}}%
	\IfBooleanTF{#1}
		{%
			\!\left[#3%
			\IfValueT{#4}%
				{\midcond #4}%
			\right]
		}{%
			[#3%
			\IfValueT{#4}%
				{\cond #4}%
			]%
		}%
	}%
}

\NewDocumentCommand\variance{somg}{
	{\mathbb{V}%
	\IfValueT{#2}%
		{_{#2}}%
	\IfBooleanTF{#1}
		{%
			\!\left[#3%
			\IfValueT{#4}%
				{\midcond #4}%
			\right]
		}{%
			[#3%
			\IfValueT{#4}%
				{\cond #4}%
			]%
		}%
	}%
}

\NewDocumentCommand\kldiv{smm}{
	{\mathit{KL}%
	\IfBooleanTF{#1}
		{%
			\left[#2%
			\IfValueT{#3}%
				{\middle\Vert #3}%
			\right]
		}{%
			[#2%
			\IfValueT{#3}%
				{\Vert #3}%
			]%
		}%
	}%
}

\def\*#1{\mathbf{#1}}

\def\bxi{\boldsymbol{\xi}}
\def\bpsi{\boldsymbol{\psi}}

\NewDocumentCommand\BracketsNoStar{mg}{
	(#1%
	\IfValueT{#2}%
		{\cond #2}%
	)%
}
\NewDocumentCommand\BracketsStar{mg}{
	\mathopen{}\left(#1%
	\IfValueT{#2}%
		{\midcond #2}%
	\right)\mathclose{}%
}
\NewDocumentCommand\pdf{smmg}{
	{#2%
	\IfBooleanTF{#1}{%
		\BracketsStar{#3}{#4}
	}{%
		\BracketsNoStar{#3}{#4}
	}}%
}

\newcommand{\npdf}[2]{\pdf{\mathcal{N}}{#1}{#2}}

\newcommand\numberthis{\stepcounter{equation}\tag{\theequation}}
\newcommand{\MethodName}{\textsc{NDFlow}}
\newcommand{\Nyul}{\textsc{Nyul}}

\title{Nonparametric Density Flows\\for MRI Intensity Normalisation}

\newcommand{\corrauth}{\textsuperscript{(\Letter)}}
\author{Daniel C. Castro\corrauth \and Ben Glocker}
\authorrunning{D.C. Castro and B. Glocker}
\institute{Biomedical Image Analysis Group\\
Imperial College London, UK\\
\email{\{dc315,b.glocker\}@imperial.ac.uk}}

\begin{document}

\mainmatter 
\maketitle


\begin{abstract}

With the adoption of powerful machine learning methods in medical image analysis, it is becoming increasingly desirable to aggregate data that is acquired across multiple sites. However, the underlying assumption of many analysis techniques that corresponding tissues have consistent intensities in all images is often violated in multi-centre databases. We introduce a novel intensity normalisation scheme based on density matching, wherein the histograms are modelled as Dirichlet process Gaussian mixtures. The source mixture model is transformed to minimise its $L^2$ divergence towards a target model, then the voxel intensities are transported through a mass-conserving flow to maintain agreement with the moving density. In a multi-centre study with brain MRI data, we show that the proposed technique produces excellent correspondence between the matched densities and histograms. We further demonstrate that our method makes tissue intensity statistics substantially more compatible between images than a baseline affine transformation and is comparable to state-of-the-art while providing considerably smoother transformations. Finally, we validate that nonlinear intensity normalisation is a step toward effective imaging data harmonisation.

\end{abstract}


\section{Introduction}

Many medical image analysis methods rely on the hypothesis that corresponding anatomical structures present similar intensity profiles. Unlike computed tomography, magnetic resonance imaging does not produce scans in an absolute standard scale, in general. Even when using the same imaging protocols, there can be significant variation between different scanners. Acquisition parameters have a complex effect on the luminance of the acquired images, therefore a simple linear rescaling of intensities is usually insufficient for effective data harmonisation \cite{Hellier2003}. Therefore, a crucial factor for enabling the construction of large-scale image databases from multiple sites is accurate nonlinear intensity normalisation.

A number of different approaches have been introduced for this task (cf.~\cite{Bergeest2008}), the most widely-adopted of which is that of Ny\'ul et al.\ \cite{Nyul2000}. The authors proposed to normalise intensities by matching a set of histogram quantiles, using these as landmarks for a piecewise linear transformation. Despite its apparent simplicity, it has proven very effective in clinical applications \cite{Shah2011}.

Our proposed method, nonparametric density flows (\MethodName), is perhaps conceptually closest to \cite{Hellier2003}, which involves matching Gaussian mixture models (GMMs) fitted to a pair of image histograms. The author used a finite mixture to represent a pre-defined set of five tissues classes, whereas we propose to use nonparametric mixtures, focussing on accurately modelling the density rather than discriminating tissue types, and sidestepping the problem of pre-selecting the number of components. A further difference to our work is that, instead of polynomially interpolating between the means of corresponding components, we build a smooth transformation model based on density flows.


\section{Method}
\label{sec:method}

We begin by justifying and describing the density model used to represent the intensity distributions to be matched. We then introduce the chosen objective function with its gradients for optimisation. Finally, we present our flow-based transformation model, which deforms the data so it conforms to the matched density model. Note that we focus here on single-modality intensity normalisation, although the entire formulation below extends naturally to the multivariate case.

\begin{figure}[tb]
	\centering
    \includegraphics[width=\textwidth]{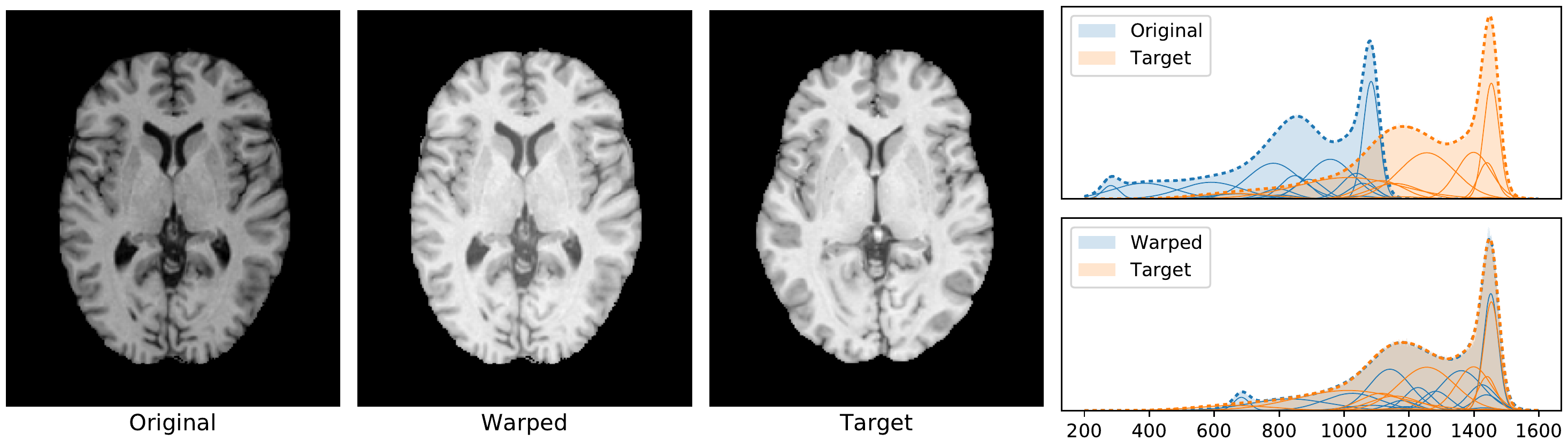}
    \caption{Comparison of two MRI scans, before and after the proposed \MethodName{} normalisation. \emph{Right:} histograms (shaded) and fitted mixture models (\emph{dotted:} likelihood, \emph{solid:} mixture components).}
    \label{fig:example_imgs_hist}
\end{figure}

\subsection{Intensity Model}\label{sec:method_model}
In order to be able to match the intensity distributions of a pair of images, a suitable probability density model is required. Typically, finite mixture models are considered for this task \cite{Hellier2003,Roy2007}. However, a well-known limitation of these is the requirement to specify a priori a fixed number of components, which may in addition call for an iterative model selection loop (e.g.~\cite{Roy2007}).

On the opposite end of the spectrum, another approach is to use kernel density estimation, which is widespread for shape registration (e.g.~\cite{Jian2011,Hasanbelliu2011}). However, this formulation would result in an unwieldy optimisation problem, involving thousands or millions of parameters and all pairwise interactions. Furthermore, the derived transformation would likely not be satisfactorily smooth without additional regularisation.

To overcome both issues we propose to use Dirichlet process Gaussian mixture models (DPGMMs) \cite{Ferguson1983}. Instead of specifying a fixed number of components, they rely on a vague concentration parameter, which regulates the expected amount of clustering fragmentation and enables them to adapt their complexity to the data at hand. By allowing an unbounded number of components and setting a versatile prior on the mixture proportions, they appear as a parsimonious middle ground for flexibility and tractability.

We fit the DPGMMs to each image's intensities using variational inference \cite{Blei2006}. More specifically, we implemented an efficient weighted variant to fit a mixture directly to each 1D histogram.

\subsection{Density Matching}
The first step is to perform a coarse affine alignment by matching the moving density's first and second moments to the target's, accounting for arbitrary translation and rescaling of the values. This same affine transformation is then also applied to the data before the nonlinear warping takes place.

We quantify the disagreement between two probability density functions $q$ and $p$ on a probability space $\mathcal{X}$ by means of the $L^2$ divergence:
\begin{equation}\label{eq:divergence}
	D_{L^2}[q, p] = \tfrac12 \norm{q - p}^2 = \tfrac12 \norm{q}^2 + \tfrac12 \norm{p}^2 - \innerprod{q,p} \,,
\end{equation}
where $\innerprod{q,p} = \int q(x)\,p(x) \diff x$ is the $L^2$ inner product and $\norm{q}=\sqrt{\innerprod{q,q}}$ is its induced norm. Aside from being symmetric, this quantity is positive and reaches zero iff $q \overset{\text{a.e.}}{=} p$. Crucially, unlike the usual Kullback--Leibler divergence, it is expressible in closed form for Gaussian mixture densities.

Let $q = \sum_k \pi_k q_k$ and $p = \sum_m \tau_m p_m$ denote two Gaussian mixtures, with components $q_k(x) = \npdf{x}{\mu_k, \lambda_k\inv}$ and $p_m(x) = \npdf{x}{\nu_m, \omega_m\inv}$. \Cref{eq:divergence} has tractable gradients w.r.t.\ the parameters of $q$ (\cref{app:div_grads}), which we use to optimise its components' means $\{\mu_k\}_k$ and precisions $\{\lambda_k\}_k$.

We have found, in practice, that it is largely unnecessary to adapt the mixing proportions, $\{\pi_k\}_k$, to get an excellent agreement between mixture densities. In fact, changing the mixture weights would require transferring samples \emph{between} mixture components. Although surely possible, we point out that in the context of histogram matching this would imply altering their semantic value (e.g.\ consider a mixture of two well-separated components representing different tissue types).

\subsection{Warping}
After matching one GMM to another, we also need a way to transform the data modelled by that GMM so it matches the target data. To this end, we draw inspiration from fluid mechanics and define the warping transformation, $f$, as the trajectories of particles under the effect of a velocity field $u$ over time, taking the probability density $q$ for the mechanical mass density. The key property that such flow must satisfy is \emph{conservation of mass}: $\partial_t q + \partial_x (q u) = 0$, where $t \mapsto q^{(t)}$ is specified directly from the density matching.

Let us first consider the case of warping a single mixture component. A random variable $x \sim \mathcal{N}(\mu_k, \lambda_k\inv)$ can be expressed via a diffeomorphic reparametrisation of a standard Gaussian, with $x = \psi_k(\epsilon) = \mu_k + \epsilon / \sqrt{\lambda_k}$ and $\epsilon \sim \mathcal{N}(0, 1)$. Assuming its mean and precision are changing with rates $\dot\mu_k$ and $\dot\lambda_k$, respectively, we can introduce a velocity field $u_k = \dot\psi_k \circ \psi_k\inv$ for its samples so that they agree with this evolving density. The instantaneous velocity at `time' $t$ is thus given by
\begin{equation}\label{eq:component_velocity}
	u_k^{(t)}(x) = \dot\mu_k^{(t)} - \frac{\dot\lambda_k^{(t)}}{2 \lambda_k^{(t)}} \big(x - \mu_k^{(t)}\big) \,.
\end{equation}

In the case of a mixture with constant weights $\{\pi_k\}_k$, we can construct a smooth, \emph{mass-conserving} global velocity field $u$ as
\begin{equation}\label{eq:mixture_velocity}
	u^{(t)}(x) = \sum_k \frac{\pi_k q_k^{(t)}(x)}{q^{(t)}(x)} \, u_k^{(t)}(x) \,,
\end{equation}
which is simply a point-wise convex combination of each component's velocity field, $u_k$, weighted by the corresponding posterior assignment probabilities.

Finally, the warping transformation $f^{(t)}$ is given by the solution to the following ordinary differential equation (ODE):
\begin{equation}\label{eq:warping_ode}
	\partial_t f^{(t)}(x) = u^{(t)}(f^{(t)}(x)) \,, \quad f^{(0)}(x) = x \,.
\end{equation}
With $f$ defined as above, we can prove that $q^{(t)}$ is indeed the density of samples from $q^{(0)}$ transformed through $f^{(t)}$, i.e.\ ${q^{(0)} = |\partial_x f^{(t)}| \, q^{(t)} \circ f^{(t)}}$ (\cref{app:mass_cons_proof}). Crucially, the true solution to \cref{eq:warping_ode} is diffeomorphic by construction, and can be numerically approximated (and inverted) with arbitrary precision. In particular, we employ the classic fourth-order Runge--Kutta ODE solver (RK4).

Now assume we obtain optimal parameter values $\{\mu_k^*\}_k$ and $\{\lambda_k^*\}_k$ after matching $q$ to $p$. We can then warp the data using the above approach, for example linearly interpolating the intermediate parameter values, ${\mu_k^{(t)} = t \mu_k^* + (1-t) \mu_k^{(0)}}$ and ${\lambda_k^{(t)} = t \lambda_k^* + (1-t) \lambda_k^{(0)}}$, hence setting the rates in \cref{eq:component_velocity} to constant values, ${\dot\mu_k = \mu_k^* - \mu_k^{(0)}}$ and ${\dot\lambda_k = \lambda_k^* - \lambda_k^{(0)}}$, and integrating \cref{eq:warping_ode} for $t \in [0, 1]$.

\subsection{Practical Considerations}
Since each medical image in a dataset can have millions of voxels, computing the posteriors and flows for every voxel individually can be too expensive for batch processing. To mitigate this issue, we can compute the end-to-end transformation on a mesh in the range of interest, which is then interpolated for the intensities in the entire volume. In the reported experiments, we have used a uniformly-spaced mesh of 200 points, which has proven accurate enough for normalisation purposes.

Note that the transformation could also be computed on the histogram of discrete intensity values and built into a look-up table. However, this would not scale well to two or more dimensions for multi-modal intensity normalisation, whereas a mesh would not need to be very fine nor require a regular grid layout.


\section{Experiments}
\label{sec:experiments}

\subsection{Dataset}
Our experiments were run on 581 T1-weighted MRI scans from the IXI database, collected from three imaging centres with different scanners.\footnote{\url{http://brain-development.org/ixi-dataset/}} Each scan was bias field-corrected using SPM12\footnote{\url{http://www.fil.ion.ucl.ac.uk/spm/software/spm12/}} with default settings and rigidly registered to MNI space. SPM12 was further used to produce grey matter (GM), white matter (WM) and cerebrospinal fluid (CSF) tissue probability maps. We obtained brain masks by adding the three probability maps and thresholding at \num{0.5}. The statistics reported below were weighted by the voxel-wise tissue probabilities to account for partial-volume effects and segmentation ambiguities.

\subsection{Setup}

\begin{figure}[tb]
	\centering
    \includegraphics[width=\textwidth]{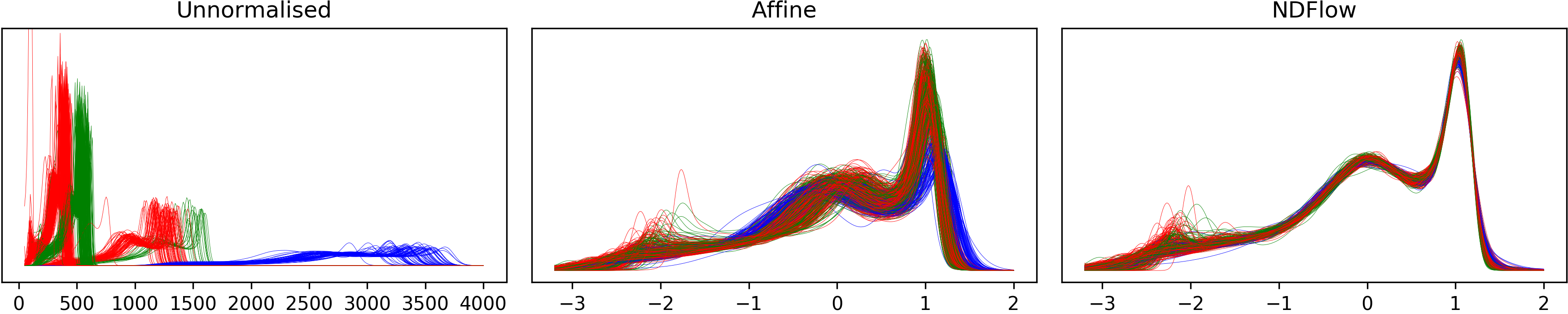}
    \caption{Population densities, colour-coded by imaging centre}
    \label{fig:pop_gmms}
\end{figure}

We firstly fitted the nonparametric mixture models to the full integer-value histograms of the raw images (inside the brain masks), as described in \cref{sec:method_model}. We set the DP's concentration parameter to 2 and used data-driven Normal--Gamma priors for the components. As an ad-hoc post-processing step, we pruned the leftover mixture components with weights smaller than $10^{-3}$. In the absence of one global reference distribution, we affinely aligned these DPGMMs and the corresponding data to zero mean and unit variance (cf.~\cref{fig:pop_gmms}, middle).

After this rough alignment, global and centre-wise average densities were computed. These were then considered as histograms to which we fitted global and centre-wise reference DPGMMs.

For normalisation, we consider two scenarios. The first is to normalise each centre's reference distribution to the global target, then to apply this same transformation to all subjects in that centre. In the other approach, each subject's image is individually normalised to the global target density. These scenarios reflect different practical applications where the centre-wise normalisation aims to preserve intra-centre variation, which might be desired. On the other hand, the individual normalisation aims to make all scans as similar as possible.

We compare our technique to Ny\'ul et al.'s prevalent quantile-based, piecewise linear histogram matching method \cite{Nyul2000}, considered state-of-the-art for intensity normalisation and referred here as \Nyul. We acquired the default 11 landmarks (histogram deciles and upper/lower percentiles) from the affine-aligned data for all subjects, then normalised each subject to this set of average landmarks.

\subsection{Results}

\paragraph{Histogram Fitness.}
\begin{figure}[tb]
	\centering
	\begin{subfigure}{0.32\textwidth}
		\includegraphics[width=\linewidth]{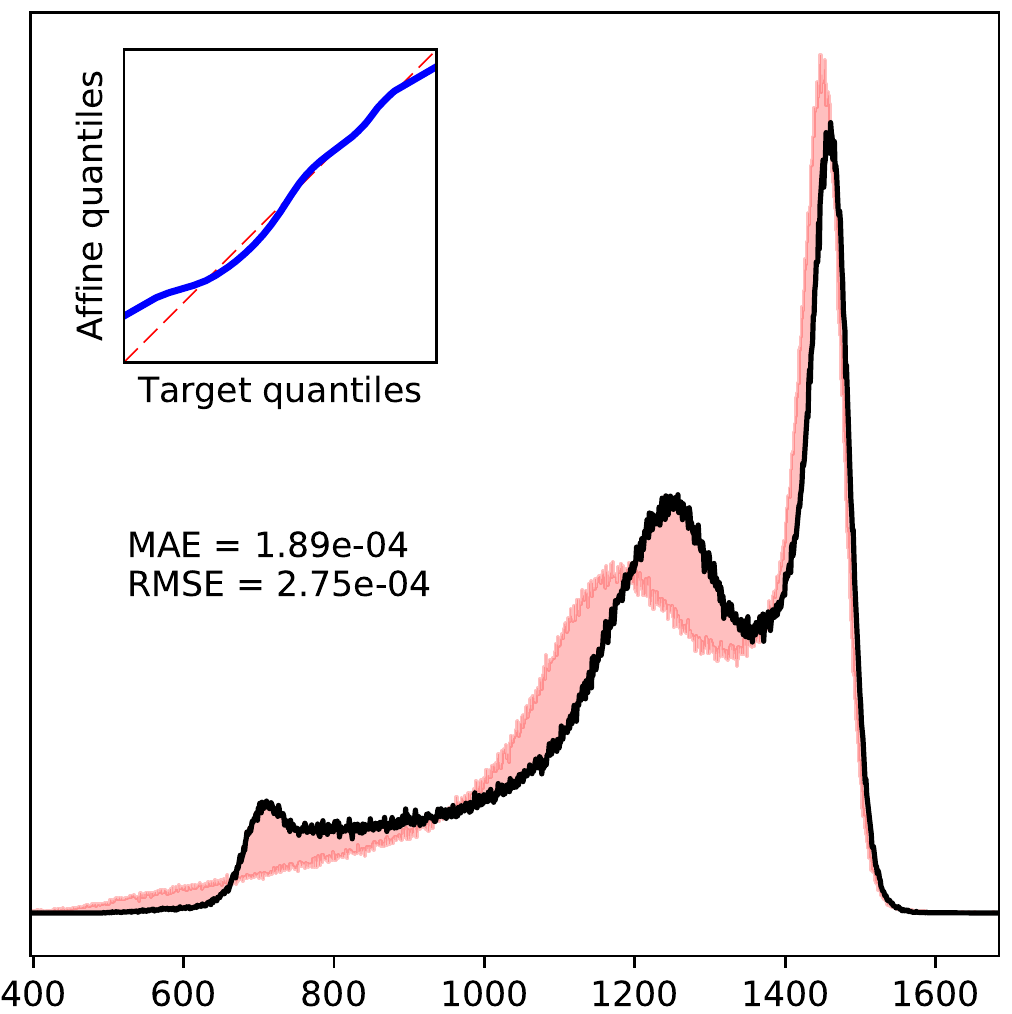}
		\captionof{figure}{Affine}
	\end{subfigure}
    \hfill
	\begin{subfigure}{0.32\textwidth}
		\includegraphics[width=\linewidth]{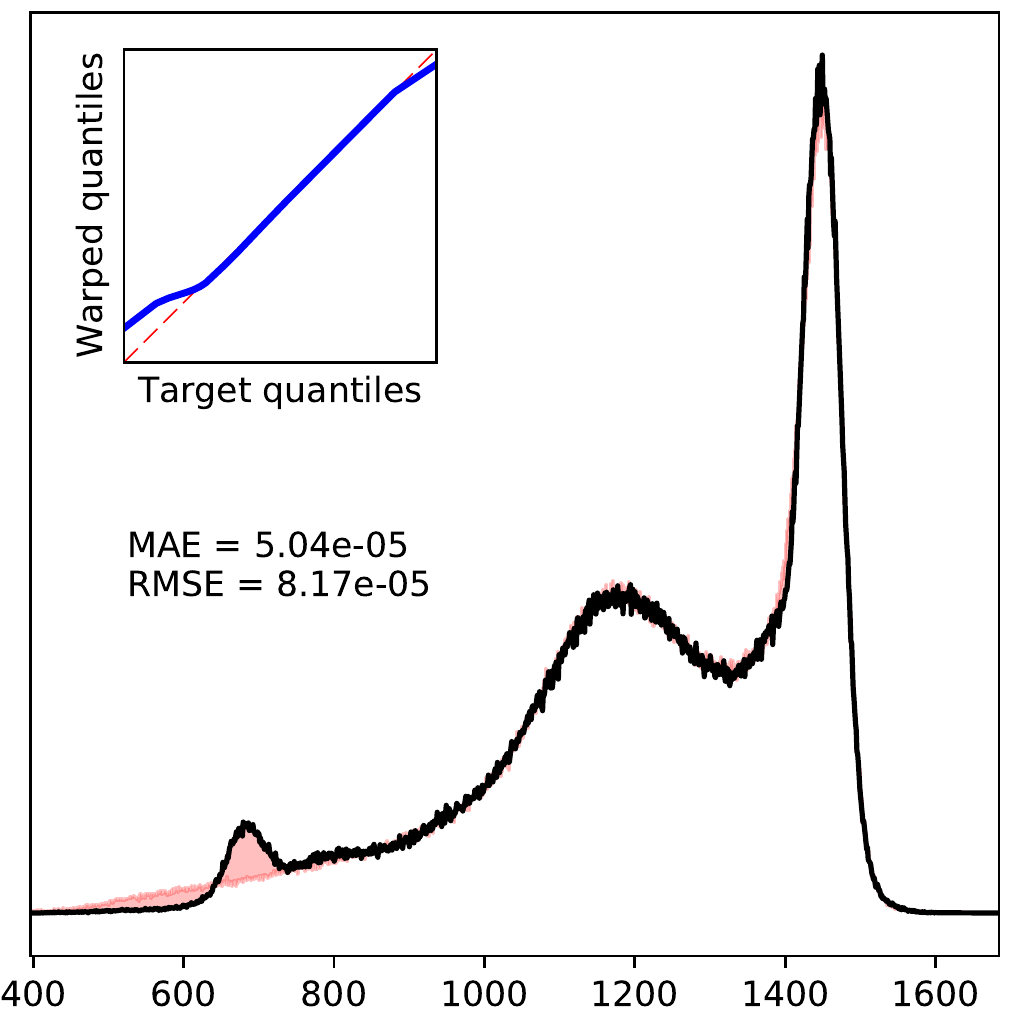}
		\captionof{figure}{\MethodName}
	\end{subfigure}
    \hfill
	\begin{subfigure}{0.32\textwidth}
		\includegraphics[width=\linewidth]{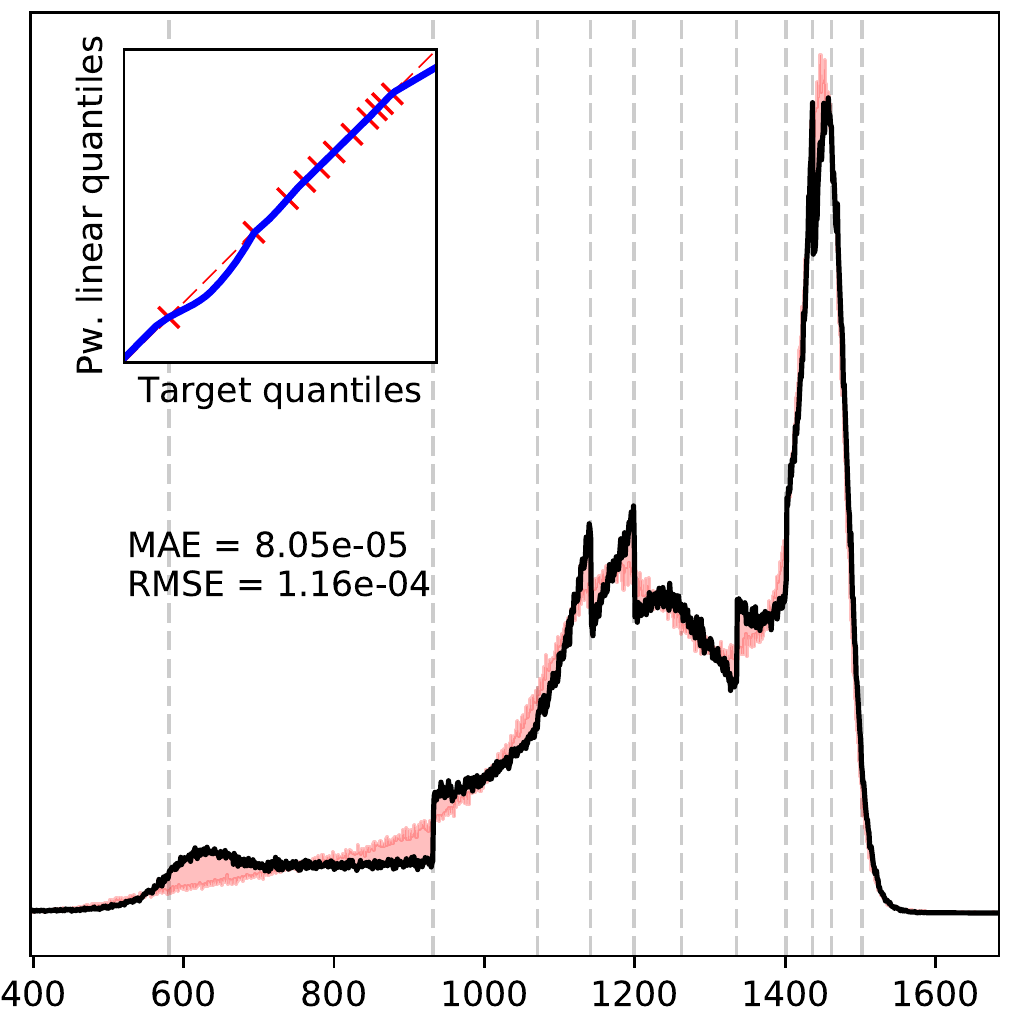}
		\captionof{figure}{\Nyul}\label{fig:hist_qq_nyul}
	\end{subfigure}
    \caption{Histograms and Q--Q plots of each of the methods against the target histogram. The shading shows the discrepancy between the transformed (black) and target histogram (light red). In the rightmost plot, the landmarks are indicated by vertical lines in the histogram and ticks in the Q--Q plot.}
    \label{fig:hist_qq}
\end{figure}

\Cref{fig:hist_qq} illustrates the results of normalisation between the pair of images in \cref{fig:example_imgs_hist}, which have a notable dissimilarity in the CSF region of the histograms. We observe that both our \MethodName- and \Nyul-transformed histograms present substantially lower mean absolute and root mean squared errors (MAE and RMSE) than the affine-aligned one, and our method performed best by a small margin. This is confirmed in a number of trials with other images.

A noteworthy artefact of \Nyul\ are abrupt jumps produced at the landmark values (e.g.\ \cref{fig:hist_qq_nyul}), which appear because interval are uniformly compressed or dilated by different factors, and may be detrimental to downstream histogram-based tasks (e.g.\ mutual information registration). \MethodName\ causes no such discontinuities due to the smoothness of the mass-conserving flows.

\paragraph{Tissue Statistics.}

\begin{table}[tb]
\centering
\sisetup{
	table-number-alignment=center,
	detect-weight, detect-family, detect-display-math,
	detect-inline-weight=math,
	separate-uncertainty=true,
	table-figures-uncertainty=4,
	table-figures-integer=1,
	table-figures-decimal=3,
    mode=text
}
\renewrobustcmd{\bfseries}{\fontseries{b}\selectfont}
\renewrobustcmd{\boldmath}{\fontseries{b}\selectfont}
\newcommand{\MethodCentre}{\MethodName: Centre}
\newcommand{\MethodIndiv}{\MethodName: Indiv.}

\begin{threeparttable}
\setlength{\tabcolsep}{0pt}
\caption{Tissue statistics after normalisation (mean $\pm$ std.\ dev., $N = 581$)}
\label{tab:tissue_stats}

\newcommand{\PM}{~$\pm$~}
\newrobustcmd{\B}{\textbf}
\newcommand{\cpad}{\hspace{6pt}}
\newcommand{\csep}{\cpad\cpad}
\newcolumntype{L}{l<{\cpad}}
\newcolumntype{R}{>{\cpad}r}
\begin{tabular}{@{\cpad}l@{\csep}l@{\cpad}RLRLRL}
  \toprule
  		& Method
        & \multicolumn{2}{c}{\hphantom{$-$}1\textsuperscript{st} Quartile}
        & \multicolumn{2}{c}{\hphantom{$-$}Median}
        & \multicolumn{2}{c}{\hphantom{$-$}3\textsuperscript{rd} Quartile}	\\
  \midrule
  WM	& Affine		&  0.900 & \PM 0.040	&  1.024 & \PM 0.045	&  1.126 & \PM 0.055	\\[4pt]
		& \MethodCentre	&  0.898 & \PM 0.040	&  1.020 & \PM \B{0.040}&  1.121 & \PM \B{0.043}\\
		& \MethodIndiv	&  0.890 & \PM \B{0.029}&  1.014 & \PM \B{0.018}&  1.120 & \PM \B{0.016}\\
		& \Nyul			&  0.897 & \PM \B{0.029}&  1.023 & \PM \B{0.015}&  1.126 & \PM \B{0.008}\\
  \midrule
  GM	& Affine		& $-$0.296 & \PM 0.142	  &  0.025 & \PM 0.117	  &  0.344 & \PM 0.080	  \\[4pt]
		& \MethodCentre	& $-$0.297 & \PM \B{0.139}&  0.025 & \PM \B{0.114}&  0.344 & \PM \B{0.076}\\
		& \MethodIndiv	& $-$0.312 & \PM \B{0.094}&  0.027 & \PM \B{0.065}&  0.351 & \PM \B{0.058}\\
		& \Nyul			& $-$0.309 & \PM \B{0.106}&  0.027 & \PM \B{0.070}&  0.350 & \PM \B{0.064}\\
  \midrule
  CSF	& Affine		& $-$2.036 & \PM 0.145	  & $-$1.486 & \PM 0.140	& $-$1.024 & \PM 0.156	\\[4pt]
		& \MethodCentre	& $-$2.035 & \PM \B{0.143}& $-$1.480 & \PM 0.142	& $-$1.018 & \PM 0.160	\\
		& \MethodIndiv	& $-$2.031 & \PM \B{0.136}& $-$1.484 & \PM 0.170	& $-$1.028 & \PM 0.191	\\
		& \Nyul			& $-$2.025 & \PM \B{0.111}& $-$1.474 & \PM 0.178	& $-$1.029 & \PM 0.207	\\
\bottomrule
\end{tabular}
\begin{tablenotes}
	\item Bold: $p < .01$, one-tailed Brown--Forsythe test for lower variance than `Affine'
\end{tablenotes}
\end{threeparttable}
\end{table}

In \cref{tab:tissue_stats} we report the WM, GM and CSF intensity statistics for different normalisations. Firstly, we see that the centre-wise normalisation had a small but significant effect on the overall distribution statistics. More importantly, the variances of the statistics after individual \MethodName\ and \Nyul\ transformations were typically similar, and both were almost always substantially smaller than the variance after only affine alignment, with the exception of CSF.

It is known that the amount of intra-cranial fluid can vary substantially due to factors such as age and some neurodegenerative conditions, and this reflects on the distributions of intensities in brain MRI scans, which is evident in \cref{fig:pop_gmms}. As a result, normalising all subjects to a `mean' distribution fails to identify a consistent reference range for CSF intensities.

A fundamental limitation of any histogram matching scheme is that it is unclear how to proceed when the distributions are \emph{genuinely} different. Intensity distributions can be strongly affected by anatomical differences; for example, we can observe large variations in the amounts of fluid and fat in brain or whole-body scans, which may heavily skew the overall distributions (moderate example in \cref{fig:hist_qq}). The underlying assumption of these methods (including ours) is that the distributions are similar enough up to an affine rescaling and a mild nonlinear deformation of the values, thus handling histograms of truly different shapes remains an open challenge. For images with different fields of view, it may be beneficial to perform image registration before applying intensity normalisation.

\paragraph{Centre Classification.}
To evaluate the effectiveness of intensity normalisation for data harmonisation, we conducted a centre discrimination experiment with random forest classifiers trained on the full images. We report the pooled test results from two-fold cross validation (detailed results in \cref{app:centre_clf_results}).

Relative to affine normalisation, centre-wise and individual \MethodName\ and \Nyul\ showed a slight drop in overall classification accuracy (94.1\% vs.\ 92.7\%, 93.6\%, 92.9\%, resp.). On the other hand, the uncertainty, as measured by the entropy of the predictions, was significantly higher (paired $t$-test, all $p<.01$). Nonlinear intensity normalisation therefore seems to successfully remove some of the biasing factors which are discriminative of the origin of the images.


\section{Conclusion}
\label{sec:conclusion}

In this paper, we have introduced a novel method for MRI intensity normalisation, called \emph{nonparametric density flows} (\MethodName). It is based on fitting and matching Dirichlet process Gaussian mixture densities, by minimising their $L^2$ divergence, and on mass-conserving flows, which ensure that the empirical intensity distribution agrees with the matched density model.

We demonstrated that our normalisation approach makes tissue intensity statistics significantly more consistent across subjects than a simple affine alignment, and compares favourably to the state-of-the-art method of Ny\'ul et al.\ \cite{Nyul2000}. We have additionally verified that \MethodName\ is able to accurately match histograms without introducing spurious artefacts produced by the competing method. Finally, we argued that both normalisation techniques can reduce some discriminative scanner biases, in a step toward effective data harmonisation.


By employing nonparametric mixture models, we are able to represent arbitrary histogram shapes with any number of modes. In addition, our formulation has the flexibility to match only part of the distributions, by freezing the parameters of some mixture components. This may be useful for ignoring lesion-related modes (e.g.\ multiple sclerosis hyperintensities), if the corresponding components can be identified (e.g., via anomaly detection). Evaluating this approach and its robustness against lesion load is a compelling direction for further research.

\subsubsection*{Acknowledgements.}
This project was supported by CAPES, Brazil (BEX 1500/2015-05), and by the European Research Council under the EU's Horizon 2020 programme (grant agreement No 757173, project MIRA, ERC-2017-STG).

\bibliographystyle{splncs03}
\bibliography{library}

\appendix
\newcommand{\pushfwd}[2]{{#1}_\# #2}
\renewcommand{\div}{\nabla\!\cdot}
\newcommand{\grad}[1][]{\nabla_{\!#1}}
\newcommand{\dx}[1]{\frac{\partial #1}{\partial x}}
\newcommand{\dt}[1]{\frac{\partial #1}{\partial t}}

\appendix

\section{Divergence Gradients}
\label{app:div_grads}

It can be shown that the derivative of the $L^2$ divergence between densities $q$ and $p$ with respect to some parameter $\theta$ of $q$ is given by
\begin{equation}\label{eq:dcs_deriv}
	\partial_\theta D_{L^2}[q, p] = \innerprod{q, \partial_\theta q} - \innerprod{p, \partial_\theta q} \,.
\end{equation}

\sloppypar{
Let $q = \sum_k \pi_k q_k$ and $p = \sum_m \tau_m p_m$ denote two Gaussian mixtures, with components $q_k(x) = \npdf{x}{\mu_k, \lambda_k\inv}$ and $p_m(x) = \npdf{x}{\nu_m, \omega_m\inv}$. Given that the derivatives of $q$ w.r.t.\ its component parameters are $
	{\frac{\partial q(x)}{\partial \mu_k} = \pi_k q_k(x) \cdot \lambda_k (x - \mu_k)}
$ and $
	{\frac{\partial q(x)}{\partial \lambda_k} = \pi_k q_k(x) \cdot \frac12 [\lambda_k\inv - (x - \mu_k)^2]}
$, the gradients of the divergence can be written as}
\begin{equation}\label{eq:grad_mean}
	\frac{\partial D_{L^2}[q, p]}{\partial \mu_k} =
    	\sum_l w_{lk} \frac{\mu_l - \mu_k}{\lambda_l\inv + \lambda_k\inv}
    	-\sum_m v_{mk} \frac{\nu_m - \mu_k}{\omega_m\inv + \lambda_k\inv} \,,
\end{equation}
\begin{multline}\label{eq:grad_prec}
	\frac{\partial D_{L^2}[q, p]}{\partial \lambda_k} =
    	\sum_l \frac{w_{lk}}{2} \mathopen{}\left[ \lambda_k\inv - (\lambda_l+\lambda_k)\inv
        	- \left( \lambda_l \frac{\mu_l - \mu_k}{\lambda_l + \lambda_k} \right)^2 \right] \\
    	-\sum_m \frac{v_{mk}}{2} \mathopen{}\left[ \lambda_k\inv - (\omega_m+\lambda_k)\inv
        	- \left( \omega_m \frac{\nu_m - \mu_k}{\omega_m + \lambda_k} \right)^2 \right] \,,
\end{multline}
where $w_{lk} = \pi_l \pi_k \innerprod{q_l, q_k}$ and $v_{mk} = \tau_m \pi_k \innerprod{p_m, q_k}$. To make sure the precisions $\{\lambda_k\}_k$ remain non-negative throughout the optimisation, we can simply reparametrise them as $\lambda_k = \ell_k^2$, with $\ell_k\in \bbbr$.

\section{Mass Conservation Proof}
\label{app:mass_cons_proof}

After performing the divergence minimisation scheme described above, we have access to the sequences $(\theta_t)_{t=0}^T$ and $(\dot{\theta}_t)_{t=0}^T$ (let us assume $t$ continuous for now), where $\dot{\theta}_t = -\partial_\theta D_{L^2}[q_t, p]$. This allows us to evaluate $q_t = q(\cdot; \theta_t)$ and $\partial_t q_t = \innerprod{\grad[\theta] q_t, \dot{\theta}_t}$. We seek a map $\bxi_t$ under which samples from $q_0$ will conform with $q_t$, i.e.
\[
	\*x \sim q_0 \implies \bxi_t(\*x) \sim q_t \,,
\]
aiming at making the final $\bxi_T(\*x)$ approximately agree with the target density $p$.

One key concept in the following developments is that of conservation of mass, a cornerstone of fluid dynamics. Taking a probability density for the typical mechanical density, the conservation of (probability) mass principle states that the probability of a point being in a fixed region of space changes by the net probability influx through its boundary. Alternatively, stated from the Lagrangian perspective, the probability of a moving region remains constant as its boundary is transported by a velocity field. Under smoothness assumptions on densities and velocities, both pictures are equivalent to the differential conservation law.

\begin{definition}
	A velocity field $(\*u_t)_t$ is said to \emph{conserve mass} for an evolving family of densities $(q_t)_t$ iff it satisfies the \emph{continuity equation}:
	\begin{equation}\label{eq:continuity}
		\forall t \,, \quad \partial_t q_t + \div (q_t \*u_t) = 0 \,.
	\end{equation}
\end{definition}

Let $\bxi_t(\*x)$ denote the trajectory of point $\*x$, as it is transported by the flow $(\*u_\tau)_\tau$ from time $0$ until $t$, with $\bxi_0(\*x) = \*x$. It can be formulated as the following ordinary differential equation (ODE):
\begin{equation}\label{eq:ode}
	\partial_t \bxi_t(\*x) = \*u_t(\bxi_t(\*x)) \,, \quad \bxi_0(\*x) = \*x \,.
\end{equation}

If a (locally Lipschitz-continuous) velocity field $\*u_t$ satisfies the continuity equation for an evolving density $q_t$, then the induced flow $\bxi_t$ is uniquely defined and $q_t = \pushfwd{(\bxi_t)}{q_0}$, i.e.\ the pushforward density through $\bxi_t$ coincides with the target \cite[p.\ 15]{Villani2009}.\footnote{Here we have overloaded the notion (and notation) of pushforward measure to the corresponding density function, denoting $\pushfwd{f}{p} = \abs{\Diff f\inv} \, p \circ f\inv$ for some PDF $p$ and diffeomorphism $f$.} Therefore, if the evolution of $q_t$ is known, we only have to determine a suitable $\*u_t$.

Recall that the evolving density satisfies the \emph{Jacobian equation}:
\begin{equation}
	\forall t \,, \quad q_0 = \det{\Diff \bxi_t} \, q_t \circ \bxi_t \,,
\end{equation}
where the Jacobian determinant, $\det{\Diff \bxi_t}$, quantifies the local compression (${>}\,1$) and expansion (${<}\,1$) of the density. It is often useful to compute it explicitly, which we can do based on \cref{eq:ode}:
\begin{equation}
	\det{\Diff \bxi_t} = \exp \int_0^t \div \*u_\tau \circ \bxi_\tau \diff \tau \,.
\end{equation}

The results shown here pertain to \emph{multivariate} densities and flows, and can naturally be specialised to the one-dimensional case discussed in the main paper.

\begin{lemma}\label{thm:velocity}
	The velocity field $(\*u_t)_t$ defined as
    \begin{equation}\label{eq:reparam_velocity}
		\*u_t = \dot{\bpsi}_{\theta_t} \circ \bpsi\inv_{\theta_t} = \innerprod{\grad[\theta] \bpsi_{\theta_t}, \dot{\theta}_t} \circ \bpsi\inv_{\theta_t} \,.
	\end{equation}
    conserves mass for a density $(q_{\theta_t})_t$ built via reparametrisation $(\bpsi_{\theta_t})_t$ of a fixed density $\tilde{q}$.
\end{lemma}
\begin{proof}
With $\*u$ defined as above (we will use a simplified notation here for clarity), we have its divergence as
\[
	\div \*u = \div (\dot{\bpsi} \circ \bpsi\inv) = \trace{(\Diff \bpsi)\inv \Diff \dot{\bpsi}} \circ \bpsi\inv \,.
\]

Let $q = \pushfwd{\bpsi}{\tilde{q}} = \det{\Diff \bpsi\inv} \tilde{q} \circ \bpsi\inv$. Taking the total time derivative of $q$,
\[
	\partial_t (q \circ \bpsi) = \partial_t q \circ \bpsi + (\grad q \circ \bpsi) \cdot \partial_t \bpsi \,,
\]
we can write	
\begin{align*}
	\partial_t q \circ \bpsi
		&= \partial_t (\det{\Diff \bpsi}\inv \tilde{q}) - (\grad q \circ \bpsi) \cdot \partial_t \bpsi \\
		&= -(\det{\Diff \bpsi}\inv \tilde{q}) \trace{(\Diff \bpsi)\inv \Diff \dot{\bpsi}} - (\grad q \circ \bpsi) \cdot \dot{\bpsi} \\
		&= -(q \div \*u + \grad q \cdot \*u) \circ \bpsi \\
		&= -\div (q \*u) \circ \bpsi \,, \numberthis \label{eq:proof_conservation}
\end{align*}
where we have applied Jacobi's formula,
$
	\partial_t \det{\Diff \bpsi}\inv = -\det{\Diff \bpsi}\inv \trace{(\Diff \bpsi)\inv \Diff \dot{\bpsi}} \,,
$
and used the fact that $\partial_t \tilde{q} = 0$.

Since we have taken $\bpsi$ to be diffeomorphic (hence surjective), the result in \cref{eq:proof_conservation} must also hold over the entire image of $\bpsi$, i.e.
\[
	\partial_t q + \div (q \*u) = 0 \,. \tag*{\qed}
\]
\end{proof}

\begin{proposition}\label{thm:mass_cons_mixture}
Let $(q_k)$ denote density functions and $(\*u_k)$ velocity fields, and $(\pi_k)$ such that $\pi_k > 0$ and $\sum_k \pi_k = 1$. If the flow determined by each $\*u_k$ conserves mass for the evolution of the respective $q_k$, then the flow determined by $\*u = \sum_k \frac{\pi_k q_k}{q} \*u_k$ conserves mass for the evolution of $q = \sum_k \pi_k q_k$.
\end{proposition}
\begin{proof}
Let us assume that, for some choice of parametric density family $q_k$, we have obtained a velocity field $\*u_k$ that satisfies the continuity equation for the evolution of $q_k$, for each $k$:
\begin{equation}
	\forall k \,, \quad \dt{q_k} + \div (q_k \*u_k) = 0 \,.
\end{equation}

Now, taking a convex combination of the above with weights $(\pi_k)_k$, we obtain
\begin{align*}
	0 &= \sum_k \pi_k \left[ \dt{q_k} + \div (q_k \*u_k) \right] \\
		&= \dt{} \left( \sum_k \pi_k q_k \right) + \div \left( \sum_k \pi_k q_k \*u_k \right) \\
		&= \dt{} \left( \sum_k \pi_k q_k \right) + \div \left[ \left( \sum_k \pi_k q_k \right) \left( \sum_k \frac{\pi_k q_k}{\sum_l \pi_l q_l} \*u_k \right) \right] \\
		&= \dt q + \div (q \*u) \,,
\end{align*}
where we have defined $q = \sum_k \pi_k q_k$ and $\*u = \sum_k \frac{\pi_k q_k}{q} \*u_k$. \qed
\end{proof}

\clearpage
\section{Centre Classification Results}
\label{app:centre_clf_results}

\begin{table}[!h]
\newcommand{\MethodCentre}{\MethodName: Centre}
\newcommand{\MethodIndiv}{\MethodName: Indiv.}
\centering
\caption{Centre classification accuracy. The central columns correspond to the three London imaging centres where IXI data was collected: Guy's Hospital, Hammersmith Hospital (HH) and Institute of Psychiatry (IOP). The `Overall' column shows the class-balanced average accuracy.}
\setlength{\tabcolsep}{10pt}
\begin{tabular}{lcccc}
	\toprule
					& Guy's		& HH		& IOP		& Overall	\\
    \cmidrule(lr){2-4} \cmidrule(lr){5-5}
    Unnormalised	& 0.9906	& 0.9892	& 1.0000	& 0.9913	\\[4pt]
    Affine			& 0.9687	& 0.9135	& 0.8904	& 0.9411	\\[4pt]
    \MethodCentre	& 0.9687	& 0.9081	& 0.7945	& 0.9272	\\
    \MethodIndiv	& 0.9749	& 0.9135	& 0.8219	& 0.9359	\\
    \Nyul			& 0.9655	& 0.9189	& 0.7945	& 0.9289	\\
    \bottomrule
\end{tabular}
\end{table}

\begin{figure}[!h]
\centering
\includegraphics[width=.95\textwidth]{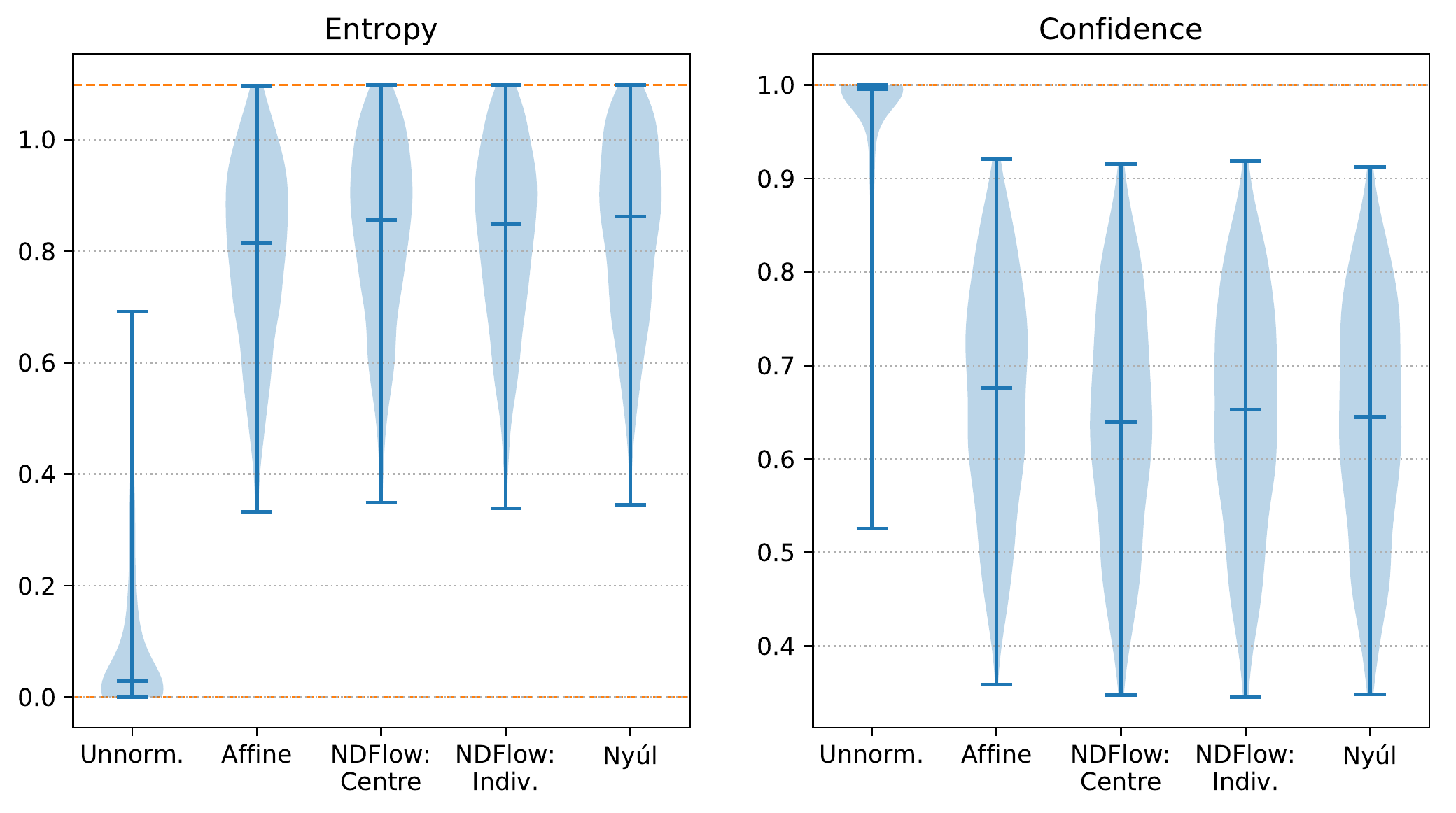}
\caption{Centre prediction statistics. Each violin plot indicates minimum, maximum and median. Dashed horizontal lines mark maximal possible values ($\log 3 \approx 1.10$ for entropy). Confidence here is the predicted probability of the chosen class.}
\end{figure}

\end{document}